\documentclass{alt2019}
\usepackage{algorithm}
\usepackage{algorithmic}
\title[LDP learning via Polynomial Approximations]{
Noninteractive Locally Private Learning of Linear Models \\ via Polynomial Approximations
}
\usepackage{times}

\usepackage{amsmath}
\usepackage{color}

 \coltauthor{\Name{Di Wang} \Email{dwang45@buffalo.edu}\\
 \addr Department of Computer Science and Engineering\\
       State University of New York at Buffalo\\
       Buffalo, NY 14260, USA 
 \AND
 \Name{Adam Smith} \Email{ads22@bu.edu}\\
 \addr Department of Computer Science \\
       Boston University\\
       Boston, MA 02215, USA
  \AND
  \Name{Jinhui Xu} \Email{jinhui@buffalo.edu}\\
   \addr Department of Computer Science and Engineering\\
       State University of New York at Buffalo\\
       Buffalo, NY 14260, USA
 }

\begin{document}

\maketitle
\begin{abstract} Minimizing a convex risk function is the main step in many basic learning algorithms. 
We study protocols for convex optimization which provably leak very little about the individual data points that constitute the loss function. Specifically, we consider differentially private algorithms that operate in the local model, where each data record is stored on a separate user device and randomization is performed locally by those devices. We give new protocols for \emph{noninteractive} LDP convex optimization---i.e., protocols that require only a single randomized report from each user to an untrusted aggregator. 

We study our algorithms' performance with respect to expected loss---either over the data set at hand (empirical risk) or a larger population from which our data set is assumed to be drawn. Our error bounds depend on the form of individuals' contribution to the expected loss. For the case of \emph{generalized linear losses} (such as hinge and logistic losses), we give an LDP algorithm whose sample complexity is only linear in the dimensionality $p$ and exponential in other terms (the privacy parameters $\epsilon$ and $\delta$, and the desired excess risk $\alpha$). This is the first algorithm for nonsmooth losses with linear dependence on $p$.


Our result for the hinge loss is based on a technique, dubbed polynomial of inner product approximation, which may be applicable to other problems. Our results for generalized linear losses 
is based on new reductions to the case of hinge loss. 
\end{abstract}


\section{Introduction}

In the big data era, a tremendous amount of individual data are generated every day. Such data, if properly used, could greatly improve
many aspects of our daily lives.
However, due to the sensitive nature of such data, a great deal of care needs to be taken while analyzing them.  
Private data analysis seeks to enable the benefits of learning from data with the guarantee of privacy-preservation.
Differential privacy \citep{dwork2006calibrating} has emerged as a rigorous notion for privacy which allows accurate data analysis with a guaranteed bound on the increase in harm for each individual to contribute her data. 
Methods to guarantee differential privacy have been widely studied, and recently adopted in  industry~\citep{208167,erlingsson2014rappor}.\par

Two main user models have emerged for differential privacy: the central model and the local one. In the central model, data are managed by a trusted central entity which is responsible for collecting them and for deciding which differentially private data analysis to perform and to release. A classical use case for this model is the one of census data~\citep{Haney:2017}. In the local model instead, each individual manages his/her proper data and discloses them to a server through some differentially private mechanisms. The server collects the (now private) data of each individual and combines them into a resulting  data analysis. A classical use case for this model is the one aiming at collecting statistics from user devices like in the case of Google's Chrome browser~\citep{erlingsson2014rappor}, and  Apple's iOS-10 \citep{208167,DBLP:journals/corr/abs-1709-02753}.

In the local model, two basic kinds of protocols exist: interactive and non-interactive.  \cite{smith2017interaction} have recently investigated the power of non-interactive differentially private protocols. These protocols are more natural for the classical use cases of the local model, e.g., both the projects from Google and Apple use the non-interactive model. Moreover, implementing efficient interactive protocols in such applications is more challenging due to the latency of the network. Despite its applications in industry, the local model has been much less studied than the central one. Part of the reason for this is that there are intrinsic limitations in what one can do in the local model. As a consequence, many basic questions, that are well studied in the central model, have not been completely understood in the local model, yet. 

In this paper, we study differentially private  Empirical Risk Minimization in the  non-interactive local model.  Before showing our contributions and  discussing comparisons with previous work, we first discuss our motivations. 
\paragraph{Problem Setting} Given a  convex, closed and bounded constraint set $\mathcal{C}\subseteq \mathbb{R}^{p}$, a data universe $\mathcal{D}$, and a loss function $\ell:\mathcal{C}\times \mathcal{D}\mapsto \mathbb{R}$, a dataset $D=\{(x_1, y_1), (x_2, y_2), \cdots, (x_n, y_n)\}\in \mathcal{D}^n$ with data records $\{x_i\}_{i=1}^n\subset \mathbb{R}^p$ and labels (responses) $\{y_i\}_{i=1}^n\subset \mathbb{R}$ defines an \emph{empirical risk} function: $L(w;D)=\frac{1}{n}\sum_{i=1}^{n}\ell(w; x_i, y_i)$ (note that in some settings, such as mean estimation, there may not be separate labels). When the inputs are drawn i.i.d from an unknown underlying distribution $\mathcal{P}$ on $\mathcal{D}$, we can also define the \emph{population risk} function: $L_\mathcal{P}(w)=\mathbb{E}_{D\sim \mathcal{P}^n}[\ell(w;D)]$. 

Thus, we have the following two types of excess risk measured at a particular output $w_{\text{priv}}$: The empirical risk, 
$$\text{Err}_{D}(w_{\text{priv}})=L(w_{\text{priv}};D)-\min_{w\in \mathcal{C}}L(w;D) \,,$$ and the population risk,
$$\text{Err}_{\mathcal{P}}(w_{\text{priv}})=L_{\mathcal{P}}(w_{\text{priv}})-\min_{w\in \mathcal{C}} L_{\mathcal{P}}(w).$$

The problem considered in this paper is to design noninteractive LDP protocols that minimize the empirical and/or population excess risks. Alternatively, we can express our goal this problem in terms of \emph{sample complexity}: find  the smallest $n$ for which we can design protocols that achieve error at most $\alpha$ (in the worst case over data sets, or over generating distributions, depending on how we measure risk). 

\citet*{duchi2013local} first considered worst-case error bounds for LDP convex optimization. For 1-Lipchitz convex losses over a bounded constraint set, they gave a highly interactive SGD-based protocol with sample complexity $n=O(p/\epsilon^2\alpha^2)$; moreover, they showed that no LDP protocol which interacts with each player only once can achieve asymptotically better sample complexity, even for linear losses. 

 \citet*{smith2017interaction} considered the round complexity of LDP protocols for convex optimization. They observed that known methods perform poorly when constrained to be run noninteractively. They gave new protocols that improved on the state of the art but nevertheless required sample complexity  exponential in $p$.  Specifically, they showed:

\begin{theorem}[\citet{smith2017interaction}] \label{theorem:1}
	Under the assumptions above, there is a noninteractive $\epsilon$-LDP algorithm  that for all distribution $\mathcal{P}$ on $\mathcal{D}$, with probability $1-\beta$, returns a solution with population error at most $\alpha$ as long as $n=\tilde{O}(c^p \log(1/\beta)/ \epsilon^{2} \alpha^{p+1})$, where $c$ is an absolute constant.
 A similar result holds for empirical risk $\text{Err}_{D}$. 
 \end{theorem}

Furthermore, lower bounds on the parallel query complexity of stochastic optimization (e.g., \citet{Nemirovski94,Woodworth2018}) mean that, for natural classes of LDP optimization protocols (based on measuring noisy gradients), the exponential dependence of the sample size on the dimension $p$ (in the terms of $\alpha^{-(p+1)}$ and $c^p$) is, in general, unavoidable \citep{smith2017interaction}.

This situation is challenging: when the dimensionality $p$ is high, the sample complexity (at least $\alpha^{-(p+1)}$) is enormous even for a very modest target error. However, several results have already shown that for some specific loss functions, the exponential dependency on the dimensionality can be avoided. For example,
 \citet{smith2017interaction} show that, in the case of linear regression, there is a noninteractive $(\epsilon,\delta)$-LDP algorithm\footnote{Note that these two results are for noninteractive $(\epsilon,\delta)$-LDP, a variant of $\epsilon$-LDP. We omit quasipolynomial terms related to $\log(1/\delta)$ in this paper.}  with expected empirical error $\alpha$ and sample complexity $n=\tilde O(p\epsilon^{-2}\alpha^{-2})$. This indicates that there is a gap between the general case and what is achievable for some specific, commonly used loss functions.
 
\paragraph{Our Contributions} The results above motivate the following basic question:

\begin{quote}\em 
Are there natural conditions on the loss function which allow for noninteractive $\epsilon$-LDP algorithms with sample complexity growing sub-exponentially (ideally,  polynomially or even linearly) on the dimensionality $p$?
\end{quote}

 To answer this question, we  first consider the case of hinge loss functions, which are ``plus functions'' of an inner product: $\ell(w; x,y) = [y\langle w,x\rangle ]_+$ where $[a]_+=\max\{0, a\}$. Hinge loss arises, for example, when fitting support vector machines. We construct our noninteractive LDP algorithm by using Bernstein polynomials to approximate the loss's derivative  after smoothing. Players randomize their inputs by randomizing the coefficients of a polynomial approximation.  The aggregator uses the noisy reports to provide biased gradient estimates when running Stochastic Inexact Gradient Descent \citep{dvurechensky2016stochastic}. 
 
 We show that a variant of the same algorithm can be applied to convex, 1-Lipschitz generalized linear loss function, any loss function where each records's contribution has the form $\ell(w; x, y)=f(y_i \langle w, x_i \rangle)$ for some $1$-Lipschitz convex function $f$.
 
 Our algorithm has sample complexity that depends only linearly, instead of exponentially, on the dimensionality $p$. The protocol exploits  the fact that any 1-dimensional $1$-Lipschitz convex function  can be expressed as a convex combination of linear functions and hinge loss functions. 
 

\section{Related Work}
 
Differentially private convex optimization, first formulated by \cite{chaudhuri2009privacy} and \citet*{chaudhuri2011differentially}, has been the focus of an active line of work for the past decade, such as  \citep{WangYX17,bassily2014private,kifer2012private,chaudhuri2011differentially,talwar2015nearly,dwang19aaai}. We discuss here only those results which are related to the local model.

\citet{kasiviswanathan2011can} initiated the study of learning under local differential privacy. Specifically, they
showed a general equivalence between learning in the local model and learning in the  statistical query model. \cite{beimel2008distributed} gave the first lower bounds for the accuracy of LDP protocols, for the special case of counting queries (equivalently, binomial parameter estimation). 
The general problem of LDP convex risk minimization was first studied by \cite{duchi2013local}, which provided tight upper and lower bounds for a range of settings. Subsequent work considered a range of statistical problems in the LDP setting, providing upper and lower bounds---we omit a complete list here. 

\cite{smith2017interaction} initiated the study of the round complexity of LDP convex optimization, connecting it to the parallel complexity of (nonprivate) stochastic optimization.


	\begin{table*}[ht]
	\begin{center}
		\resizebox{\textwidth}{!}{%
			\begin{tabular}	[h]{|c|c|c|}
				\hline
				Methods & Sample Complexity & Assumption on the Loss Function \\ 
				\hline
								\citep[Claim 4]{smith2017interaction} & $\tilde{O}(4^p\alpha^{-(p+2)}\epsilon^{-2})$ & 1-Lipschitz\\ [1.5ex]
				\hline
				 \citep[Theorem 10]{smith2017interaction} & $\tilde{O}(2^p\alpha^{-(p+1)}\epsilon^{-2})$  & 1-Lipschitz and Convex \\[1.5ex]		            \hline

				\citet{smith2017interaction} &  $\Theta(p\epsilon^{-2}\alpha^{-2})$  & Linear Regression  \\[1.5ex]
				\hline
								\cite{DBLP:journals/corr/abs-1802-04085} &  $\tilde{O}\big( (c p^{\frac{1}{4}})^p\alpha^{-(2+\frac{p}{2})}\epsilon^{-2}\big)$  & $(8, T)$-smooth \\[1.5ex]
				\hline
				\cite{DBLP:journals/corr/abs-1802-04085} &  $\tilde{O}(4^{p(p+1)}D^2_p\epsilon^{-2}\alpha^{-4})$  &  $(\infty, T)$-smooth \\ [1.5ex]				\hline
				\cite{DBLP:conf/icml/0007MW17}	 
				&  $p\cdot \left(\frac{1}{\alpha}\right)^{O(\log \log (1/\alpha) +  \log(1/\epsilon))}$ 
	& Smooth Generalized Linear \\[3ex]
	\hline
	            \textbf{This Paper} &
	            $p\cdot \left(\frac{C}{\alpha^3}\right)^{O( 1/\alpha^3)}/\epsilon^{O(\frac{1}{\alpha^3})}$  
	            & 1-Lipschitz Convex Generalized Linear \\[3ex]
	           \hline		
			\end{tabular} } 
		\caption{Comparisons on the sample complexities for achieving error $\alpha$ in the empirical risk,   where $c$ is a constant. We assume that $\|x_i\|_2, \|y_i\|\leq 1$  for every $i\in [n]$ and the constraint set $\|\mathcal{C}\|_2\leq 1$. Asymptotic statements assume $\epsilon,\delta,\alpha \in (0,1/2)$ and ignore quasipolynomial dependencies on $\log(1/\delta)$.}\label{Table1}
	\end{center}
\end{table*}

Convex risk minimization in the \emph{noninteractive} LDP  received considerable recent attention  \citep{DBLP:conf/icml/0007MW17,smith2017interaction,DBLP:journals/corr/abs-1802-04085} (see Table \ref{Table1} for details). 
\citet{smith2017interaction} first studied the problem with general convex loss functions and showed that the exponential dependence on the dimensionality is unavoidable for a class of noninteractive algorithms.  \citet{DBLP:journals/corr/abs-1802-04085} demonstrated that such an exponential dependence in the term of $\alpha$ is  avoidable 
if the loss function is smooth enough ({\em i.e.,} 
$(\infty, T)$-smooth). 
Their result even holds for non-convex loss functions. However, there is still another term $c^{p^2}$ in the sample complexity.  In this paper, we investigate the conditions which allow us to avoid this issue and obtain 
sample complexity which is linear or sub-exponential in $p$.

The work most related to ours is that of \citep{DBLP:conf/icml/0007MW17},
which also considered some specific loss functions in high dimensions, such as sparse linear regression and kernel ridge regression.  They first propose a method based on  Chebyshev polynomial approximation to the gradient function. Their idea is a key ingredient in our algorithms.  There are still several differences. First,  their analysis requires additional assumptions on the loss function, such as smoothness and boundedness of  higher order derivatives, which are not satisfied by the hinge loss. In contrast, our approach  applies to any convex, 1-Lipschitz generalized linear loss. Second, we introduce a novel argument to "lift" our hinge loss algorithms to more general linear losses. 
\section{Preliminaries}
{\bf Assumption 1}
We assume that $\|x_i\|_2\leq 1$ and $|y_i|\leq 1$ for each $i\in [n]$  and the constraint set $\|\mathcal{C}\|_2\leq 1$. Unless specified otherwise, the loss function is assumed to be general linear, that is, the loss function $\ell(\theta; x_i,y_i) \equiv f(y_i\langle x_i, \theta \rangle)$ for some 1-Lipschitz convex function.

We note that the above assumptions on $x_i, y_i$ and $\mathcal{C}$ are quite common for the studies of DP-ERM  \citep{smith2017interaction,DBLP:journals/corr/abs-1802-04085,DBLP:conf/icml/0007MW17}. The general linear assumption holds for a large class of functions such as Generalized Linear Model and SVM. We also note that there is another definition for general linear functions, $\ell(w; x,y)=f(<w,x>,y)$, which is more general than our definition. This class of functions has been studied in \citep{kasiviswanathan2016efficient,DBLP:journals/corr/abs-1802-04085};  we leave as future research to extend our work to this class of loss functions.

\paragraph{Differential privacy in the local model.} In LDP, we have a data universe $\mathcal{D}$,  $n$ players with each holding  a private data record $x_i\in \mathcal{D}$, and a server 
coordinating the protocol. An LDP protocol executes a total of $T$ rounds. In each round, the server sends a message, which is also called a query, to a subset of the players requesting them to run a particular algorithm. Based on the query, each player $i$ in the subset selects an algorithm $Q_i$, runs it on her own data, and sends the output back to the server.

\begin{definition}\citep{EvfimievskiGS03,dwork2006calibrating}\label{def:1}
An algorithm $Q$ is $\epsilon$-locally differentially private (LDP) if for all pairs $x,x'\in \mathcal{D}$, and for all events $E$ in the output space of $Q$, we have
\begin{equation*}
    \text{Pr}[Q(x)\in E]\leq e^{\epsilon}\text{Pr}[Q(x')\in E].
\end{equation*}
A multi-player protocol is $\epsilon$-LDP if for all possible inputs and runs of the protocol, the transcript of player i's interaction with the server is $\epsilon$-LDP. If $T=1$, we say that the protocol is $\epsilon$ non-interactive LDP.
\end{definition} 
\cite{kasiviswanathan2011can} gave a separation between interactive and noninteractive protocols. Specifically, they showed that there is a concept class, similarity to parity, which can be efficiently learned by interactive algorithms but which requires sample size exponential in the dimension to be learned by noninteractive local algorithms.

The following theorem shows the convergence rate of the Stochastic Inexact Gradient Method  \citep{dvurechensky2016stochastic}, which will be used in our algorithm. We first give the definition of inexact oracle.

\begin{definition}\label{def:5}
For an objective function, a $(\gamma, \beta, \sigma)$ stochastic oracle returns a tuple 
$(F_{\gamma, \beta, \sigma}(w; \xi)$, $G_{\gamma, \beta, \sigma}(w; \xi))$ such that
\begin{align*}
&\mathbb{E}_{\xi}[F_{\gamma, \beta, \sigma}(w; \xi)]=f_{\gamma, \beta, \sigma}(w),\\
&\mathbb{E}_\xi[G_{\gamma, \beta, \sigma}(w; \xi)]=g_{\gamma, \beta, \sigma}(w),\\
&\mathbb{E}_{\xi}[\|G_{\gamma, \beta, \sigma}(w; \xi)-g_{\gamma, \beta, \sigma}(w)\|_2^2]\leq \sigma^2,\\
&0\leq f(v)- f_{\gamma, \beta, \sigma}(w)- \langle g_{\gamma, \beta, \sigma}(w),v-w\rangle\leq \frac{\beta}{2}\|v-w\|^2+\gamma, \forall v,w\in\mathcal{C}.
\end{align*}
\end{definition}

\begin{lemma}[Convergence Rate of SIGM \citep{dvurechensky2016stochastic}]\label{lemma:2}
Assume that $f(w)$ is endowed with a $(\gamma, \beta, \sigma)$ stochastic oracle with $\beta\geq O(1)$.  Then, the sequence $w_k$ generated by SIGM algorithm satisfies the following inequality 
\begin{equation*}
\mathbb{E}[f(w_k)]-\min_{w\in \mathcal{C}}f(w)\leq \Theta(\frac{\beta\sigma\|\mathcal{C}\|_2^2}{\sqrt{k}}+\gamma).
\end{equation*}

\end{lemma}
\paragraph{\textbf{Bernstein polynomials and approximation}}

We give here some basic definitions that  will be used in the sequel; more details can be found in \citep{alda2017bernstein,lorentz1986bernstein,micchelli1973saturation}.

\begin{definition}\label{definition:4}
	Let $k$ be a positive integer. The Bernstein basis polynomials of degree $k$ are defined as $b_{v,k}(x)=\binom{k}{v}x^{v}(1-x)^{k-v}$ for $v=0,\cdots,k$.
\end{definition}
\begin{definition}\label{definition:5}
	Let $f:[0,1]\mapsto \mathbb{R}$ and $k$ be a positive integer. Then, the Bernstein polynomial of $f$ of degree $k$ is defined as $B_k(f;x)=\sum_{v=0}^{k}f(v/k)b_{v,k}(x)$.  
We denote by $B_k$ the Bernstein operator $B_k(f)(x)=B_k(f,x)$.
\end{definition}
Bernstein polynomials can be used to approximate some smooth functions over $[0,1]$.
\begin{definition}[\citep{micchelli1973saturation}] \label{definition:6}
	Let $h$ be a positive integer. The iterated Bernstein operator of order $h$ is defined as the sequence of linear operators $B_k^{(h)}=I-(I-B_k)^h=\sum_{i=1}^{h}\binom{h}{i}(-1)^{i-1}B_k^i$, where $I=B_k^0$ denotes the identity operator and $B_k^i$ is defined as $B_k^i=B_k\circ  B_k^{k-1}$. The iterated Bernstein polynomial of order $h$ can be computed as 
	$B_k^{(h)}(f;x)=\sum_{v=0}^{k}f(\frac{v}{k})b_{v,k}^{(h)}(x),$
	where $b^{(h)}_{v,k}(x)=\sum_{i=1}^{h}\binom{h}{i}(-1)^{i-1}B^{i-1}_k(b_{v,k};x)$.
\end{definition}
Iterated Bernstein operator can well approximate multivariate $(h,T)$-smooth functions.
\begin{definition}[\citep{micchelli1973saturation}] \label{definition:7}
	Let $h$ be a positive integer and $T>0$ be a constant. A function $f:[0,1]^{p}\mapsto \mathbb{R}$ is $(h,T)$-smooth if it is in class $\mathcal{C}^{h}([0,1]^{p})$ and its partial derivatives up to order $h$ are all bounded by $T$. We say it is $(\infty,T)$-smooth, if for every $h\in \mathbb{N}$ it is $(h,T)$-smooth.
\end{definition}
\begin{lemma}[\citep{micchelli1973saturation}] \label{lemma:13}
	If $f:[0,1]\mapsto \mathbb{R}$ is a $(2h,T)$-smooth function, then for all positive integers $k$ and $y\in[0,1]$, we have $|f(y)-B_k^{(h)}(f;y)|\leq TD_h k^{-h}$, where $D_h$ is a constant independent of $k,f$ and $y$.
\end{lemma}

\section{Main Results}

In this section, we present our main results for LDP-ERM. 

\subsection{Sample Complexity for Hinge Loss Function}
We first consider LDP-ERM with hinge loss function and then extend the obtained result to general convex linear functions.

The hinge loss function is defined as $\ell(w; x_i, y_i)=f(y_i\langle x_i, w \rangle)=[\frac{1}{2}-y_i\langle w, x_i\rangle]_+$, where the plus function $[x]_+=\max\{0, x\}$, {\em i.e.,} $f(x)=\max\{0,\frac{1}{2}-x\}$ for $x \in [-1,1]$. Note that to avoid the scenario that     $1-y_i\langle w, x_i \rangle$ is always greater than or equal to $0$, we use $\frac{1}{2}$, instead of $1$ as in the classical setting.  

Before showing our idea, we first smoothen the function $f(x)$. The following lemma shows one of the smooth functions that is close to $f$ in the domain of $[-1,1]$ (note that there are other ways to smoothen $f$; see \citep{chen1996class} for details).

\begin{lemma}\label{lemma:3}
Let $f_\beta(x)=\frac{\frac{1}{2}-x+\sqrt{(\frac{1}{2}-x)^2+\beta^2}}{2}$ be 
a function  with parameter $\beta>0$. Then, we have 
\begin{enumerate}

\item $|f_\beta(x)-f(x)|_\infty\leq \frac{\beta}{2}$, $\forall x\in \mathbb{R}.$
\item $f_\beta(x)$ is 1-Lipschitz, that is, $f'(x)$ is bounded by $1$ for $x\in \mathbb{R}$.
\item $f_\beta$ is $\frac{1}{\beta}$-smooth and convex.
\item $f'_\beta(x)$ is $(2, O(\frac{1}{\beta^2}))$-smooth if $\beta\leq 1$. 
\end{enumerate}
\end{lemma}

\begin{proof}[\textbf{Proof of Lemma \ref{lemma:3}}]
It is easy to see that items 1 is true. Item 2 is due to the following  $|f'_\beta(x)|=|\frac{-1+\frac{x-\frac{1}{2}}{\sqrt{(x-\frac{1}{2})^2+\beta^2}}}{2}|\leq 1$. Item 3 is because of the following $0\leq f''_\beta(x)=\frac{\beta^2}{( (x-\frac{1}{2})^2+\beta^2)^{\frac{3}{2}}}\leq \frac{1}{\beta}$. For item 4 we have $|f^{(3)}_\beta(x)|=\frac{3\beta^2 x}{(x^2+\beta^2)^{\frac{5}{2}}}\leq \frac{3}{\beta^2}.$ 
\end{proof}

The above lemma indicates that $f_\beta(x)$ is a smooth and convex function which well approximates  $f(x)$. 
This suggests that we can focus on $f_\beta(y_i\langle w, x_i\rangle)$, instead of $f$. Our idea is to construct a locally private $(\gamma, \beta, \sigma)$ stochastic oracle for some $\gamma, \beta, \sigma$ to approximate $f_\beta'(y_i\langle w, x_i\rangle)$ in each iteration, and then run the SIGM  step of \citep{dvurechensky2016stochastic}. By Lemma \ref{lemma:3}, we know that $f'_\beta$ is $(2, O(\frac{1}{\beta^2}))$-smooth;  thus, we can use Lemma \ref{lemma:13} to approximate $f'_\beta(x)$ via Bernstein polynomials. 
Let $P_d(x)=\sum_{i=0}^dc_i\binom{d}{i} x^i(1-x)^{d-i}$ be the $d$-th order Bernstein polynomial ($c_i=f_\beta'(\frac{i}{d}$) , where $\max_{x\in[-1,1]}|P_d(x)-f'_\beta(x)|\leq \frac{\alpha}{4}$ ({\em i.e.,} $d=c\frac{1}{\beta^2\alpha}$ for some constant $c>0$). Then, we have $\nabla_{w}\ell(w; x ,y)=f'(y\langle w, x \rangle)yx^T$, which can be approximated by $[\sum_{i=0}^dc_i\binom{d}{i}(y\langle w, x\rangle)^i(1-y\langle w, x\rangle)^{d-i}]yx^T$. The idea is that if $(y\langle w, x\rangle)^i$,  $(1-y\langle w, x\rangle)^{d-i}$ and $y x^T$ can be approximated locally differentially privately by directly adding $d+1$ numbers of independent Gaussian noises, which means it is possible to form an unbiased estimator of the term $[\sum_{i=0}^dc_i\binom{d}{i}(y\langle w, x\rangle)^i(1-y\langle w, x\rangle)^{d-i}]yx^T$. The error of this procedure can be estimated by 
Lemma \ref{lemma:2}. 
Details of the algorithm are given in Algorithm \ref{alg:1}.

\begin{algorithm}[h]
	\caption{Hinge Loss-LDP}
	\label{alg:1}
	\begin{algorithmic}[1]
		\STATE {\bfseries Input:} Player $i\in [n]$ holds data $(x_i, y_i)\in \mathcal{D}$, where $\|x_i\|_2\leq 1, \|y_i\|_2\leq 1$; privacy parameters $\epsilon, \delta$;  $P_d(x)=\sum_{j=0}^dc_i\binom{d}{j} x^j(1-x)^{d-j}$ be the $d$-th order Bernstein polynomial for the function of $f_\beta'$, where $c_i=f_\beta'( \frac{i}{d})$
		\FOR{Each Player $i\in [n]$}
		\STATE
		Calculate $x_{i,0}=x_i+\sigma_{i,0}$ and $y_{i,0}=y_i+z_{i,0}$,  where $\sigma_{i,0} \sim \mathcal{N}(0, \frac{32\log (1.25/\delta)}{\epsilon^2}I_{p})$ and $ z_{i,0}\sim \mathcal{N}(0, \frac{32\log (1.25/\delta)}{\epsilon^2})$.
		\FOR{$j=1,\cdots, d(d+1)$}
		    \STATE $x_{i, j}=x_i+ \sigma_{i, j}$, where $\sigma_{i, j}\sim \mathcal{N}(0, \frac{8\log(1.25/\delta)d^2(d+1)^2}{\epsilon^2}I_{p})$
		    \STATE $y_{i, j}=y_i+ z_{i, j}$, where $z_{i, j}\sim \mathcal{N}(0, \frac{8\log(1.25/\delta)d^2(d+1)^2}{\epsilon^2})$
	
		\ENDFOR
		\STATE Send $\{x_{i,j}\}_{j=0}^{d(d+1)}$ and $\{y_{i, j}\}_{j=0}^{d(d+1)}$ to the server.
		\ENDFOR
\FOR{the Server side}
	\FOR{$t=1, 2, \cdots, n$}
	\STATE Randomly sample $i\in [n]$ uniformly.
	\STATE Set $t_{i,0}=1$
	\FOR{$j=0, \cdots, d$}
	\STATE $t_{i, j}=\Pi_{k=jd+1}^{jd+j}	y_{i, k}<w_t, x_{i, k}>$ and $t_{i,0}=1$
	\STATE $s_{i,j} =\Pi_{k=jd+j+1}^{jd+d}(1-y_{i, k}<w_t, x_{i, k}> )$  and $s_{i,d}=1$
	\ENDFOR
	\STATE Denote $G(w_t, i)=(\sum_{j=0}^{d}c_j\binom{d}{j} t_{i, j}s_{i,j})y_{i, 0}x^T_{i, 0}$.
	\STATE Update SIGM in \citep{dvurechensky2016stochastic} by $G(w_t, i)$
	\ENDFOR
	\ENDFOR\\
	\Return $w_n$
	\end{algorithmic}
\end{algorithm}
\vspace{-0.1in}

\begin{theorem}\label{theorem:2}
For each $i\in [n]$, the term  $G(w_t, i)$ generated by Algorithm \ref{alg:1} is an $\big(\frac{\alpha}{2}, \frac{1}{\beta}, O(\frac{d^{3d}C_4^d\sqrt{p}}{\epsilon^{2d+2}}+\alpha+1)\big)$ stochastic oracle  for function $L_\beta(w;D)=\frac{1}{n}\sum_{i=1}^n f_\beta(y_i\langle x_i, w\rangle)$, where $f_\beta$ is the function in Lemma \ref{lemma:3}, where $C_4$ is some constant. 
\end{theorem}

\begin{proof}[\textbf{Proof of Theorem \ref{theorem:2}}]
For simplicity,  we  omit the term of $\delta$, which will not affect the linear dependency. Let 
$$\hat{G}(w, i)=[\sum_{j=0}^dc_j\binom{d}{j} (y_i\langle w, x_i\rangle)^j(1-y_i\langle w, x_i\rangle)^{d-j}]y_i x_i^T,$$
where $c_j=f'_\beta(\frac{j}{d} )$ and 
$$\mathbb{E}_i \hat{G}(w, i)=\frac{1}{n}\sum_{i=1}^n \hat{G}(w, i)=\hat{G}(w).$$
For the term of $G(w,i)$, the randomness comes from sampling the index $i$ and the Gaussian noises added for preserving local privacy. 

  Note that in total $\mathbb{E}_{\sigma, z, i}G(w,i)=\hat{G}(w)$, where $\sigma=\{\sigma_{i,j}\}_{j=0}^{\frac{d(d+1)}{2}}$ and $z=\{z_{i,j}\}_{j=0}^{\frac{d(d+1)}{2}}$.

 It is easy to see that $\mathbb{E}_{\sigma, z }G(w,i)=\mathbb{E}[(\sum_{j=0}^{d}c_j\binom{d}{j} t_{i, j}s_{i,j})y_{i, 0}x^T_{i, 0}\mid i]=\hat{G}(w,i)$, which is due to the fact that $\mathbb{E}{t_{i, j}}=(y_i\langle w, x_i\rangle)^j$, $\mathbb{E}{s_{i, j}}=(1-y_i\langle w, x_i\rangle)^{d-j}$ and each $t_{i,j}, s_{i,j}$ is independent. We now calculate the variance for this term with fixed $i$. Firstly, we have $\text{Var}(y_{i,0} x_{i, 0}^T)=O(\frac{p}{\epsilon^4})$. For each $t_{i,j}$, we get $$\text{Var}(t_{i, j})\leq \Pi_{k=jd+1}^{jd+j}	\text{Var}(y_{i, k})(\text{Var}(<w_i, x_{i, k}>)+(\mathbb{E}(w_i^Tx_{i, k}))^2)\leq \tilde{O}\big((C_1\frac{d(d+1)}{\epsilon^2})^{2j}\big).$$ and similarly we have $$\text{Var}(s_{i, j})\leq \tilde{O}\big( (C_2\frac{d(d+1)}{\epsilon^2})^{2(d-j)}\big).$$ Thus we have 
 \begin{equation*}
 	\text{Var}(t_{i,j}s_{i,j}) \leq \tilde{O}\big( (C_3\frac{d(d+1)}{\epsilon^2})^{2d}\big). 
 \end{equation*}

Since function $f_\beta'$ is bounded by $1$ and  $\binom{d}{j}\leq d^d$ for each $j$. In total, we have 
\begin{equation*}
\text{Var}(G(w_t,i)| i)\leq O(d \cdot d^d \cdot ( C_3\frac{d(d+1)}{\epsilon^2})^{2d} \cdot \frac{p}{\epsilon^4})=\tilde{O}\big( \frac{d^{6d}C^d p}{\epsilon^{4d+4}}\big).
\end{equation*}
Next we consider $\text{Var}(\hat{G}(w, i))$. Since 
\begin{multline*}
	\|\hat{G}(w,i)-f'_\beta(y_ix_i^Tw)y_ix_i^T\|^2_2=\|
[\sum_{j=0}^dc_j\binom{d}{j} (y_i\langle w, x_i\rangle)^j(1-y_i\langle w, x_i\rangle)^{d-j}-f'_\beta(w)]y_i x_i^T\|_2^2\\ \leq (\frac{1}{\beta^2d})^2\leq \frac{\alpha^2}{4},
\end{multline*}

we get 
\begin{multline*}
\text{Var}(\hat{G}(w, i))\leq O\big(\mathbb{E}[\|\hat{G}(w,i)-f'_\beta(y_ix_i^Tw)y_ix_i^T\|^2_2] +\mathbb{E}[\hat{G}(w)-\nabla L_\beta(w; D)\|_2^2]\\
+\mathbb{E}[\|f'_\beta(y_ix_i^Tw)y_ix_i^T-\nabla L_\beta(w; D)\|_2^2]\big) \leq O((\alpha+ 1)^2).
\end{multline*}
In total, we have $\mathbb{E}[\|G(w,i)-\hat{G}(w)\|_2^2]\leq \mathbb{E}[\|G(w,i)-\hat{G}(w,i)\|_2^2]+ \mathbb{E}[\|\hat{G}(w,i)-\hat{G}(w)\|_2^2]\leq 
\tilde{O}\big((\frac{d^{3d}C_4^d\sqrt{p}}{\epsilon^{2d+2}}+\alpha+1)^2\big).$

Also, we know that 
\begin{align*}
&L_\beta(v; D )-L_\beta(w; D)-\langle  \hat{G}(w), v-w \rangle =\\
&L_\beta(v; D)-L_\beta(w; D)-\langle \nabla L_\beta(w; D), v-w \rangle + \langle \nabla L_\beta(w;D)-G(w), v-w\rangle \\
&\leq \frac{1}{2\beta}\|v-w\|_2^2+\frac{\alpha}{2},
\end{align*}
since $L_\beta$ is $\frac{1}{\beta}$-smooth and  $|\langle \nabla L_\beta(w)-G(w), v-w \rangle |\leq \frac{\alpha}{2}$.

Thus, $G(w,i)$ is an $\big(\frac{\alpha}{2}, \frac{1}{\beta}, O(\frac{d^{3d}C_4^d\sqrt{p}}{\epsilon^{2d+2}}+\alpha+1)\big)$ stochastic oracle of $L_\beta$.
\end{proof}

From Lemmas \ref{lemma:2}, \ref{lemma:3}  and Theorem \ref{theorem:2}, we have the following sample complexity bound for the hinge loss function under the non-interactive local model.

\begin{theorem}\label{theorem:3}
For any $\epsilon>0$ and $0<\delta<1$, Algorithm 1 is $(\epsilon, \delta)$ non-interactively locally differentially private\footnote{Note that in the non-interactive local model, $(\epsilon, \delta)$-LDP is equivalent to $\epsilon$-LDP by using some protocol given in \cite{DBLP:journals/corr/abs-1711-04740}; this allows us to omit the term of $\delta$. The full sample complexity of $n$ is quasi-polynomial in $\ln(1/\delta)$.}. Furthermore, for the target  error $\alpha$, if we take $\beta=\frac{\alpha}{4}$ and $d=\frac{2}{\beta^2\alpha}=O(\frac{1}{\alpha^3})$. Then with the sample size $n=\tilde{O}(\frac{d^{6d}C^d p}{\epsilon^{4d+4}\alpha^2})$, the output $w_n$ satisfies the following inequality
\begin{equation*}
\mathbb{E} L(w_n, D)-\min_{w\in\mathcal{C}}L(w, D)\leq \alpha,
\end{equation*}
where $C$ is some constant. 
\end{theorem}

\begin{proof}[\textbf{Proof of Theorem \ref{theorem:3}}]

The guarantee of differential privacy is by Gaussian mechanism and composition theorem. 

By Theorem \ref{theorem:2} and Lemma \ref{lemma:2}, we have 
\begin{equation*}
\mathbb{E}L_\beta(w_n, D)-\min_{w\in\mathcal{C}} L_\beta(w, D)\leq O(\frac{(\frac{d^{3d}C_4^d\sqrt{p}}{\epsilon^{2d+2}}+\alpha+1)}{\beta \sqrt{n}}+\frac{1}{\beta^2 d})= O(\frac{d^{3d}C_4^d\sqrt{p}}{\epsilon^{2d+2}\beta \sqrt{n}}+\frac{\alpha}{2}).
\end{equation*}
By Lemma \ref{lemma:3}, we know that 
\begin{equation*}
\mathbb{E} L(w_n, D)-\min_{w\in\mathcal{C}}L(w, D)\leq  O(\beta+\frac{d^{3d}C_4^d\sqrt{p}}{\epsilon^{2d+2}\beta \sqrt{n}}+\frac{\alpha}{2}).
\end{equation*}
Thus, if we take $\beta=\frac{\alpha}{4}$, $d=\frac{2}{\beta^2\alpha}=O(\frac{1}{\alpha^3})$ and $n=\Omega(\frac{d^{6d}C_5^d p}{\epsilon^{4d+4}\alpha^2})$, we have   $$\mathbb{E} L(w_n, D)-\min_{w\in\mathcal{C}}L(w, D)\leq \alpha. $$
\end{proof}

\begin{remark}
Note that the sample complexity bound in Theorem \ref{theorem:3} is quite loose for parameters other than $p$. This is mainly due to the fact that we use only the basic composition theorem to ensure local differential privacy. It is possible to obtain a tighter bound by using Advanced Composition Theorem \citep{dwork2010boosting} (same for other algorithms in this paper). Details of the improvement are omit from this version. We can also extend to the population risk by the same algorithm,
the main difference is that now $G(w,i)$ is a $\big(\frac{\alpha}{2}, \frac{1}{\beta}, O(\frac{d^{3d}C_4^d\sqrt{p}}{\epsilon^{2d+2}}+\alpha+1)\big)$ stochastic oracle, where $\sigma^2=\mathbb{E}_{(x,y)
\sim \mathcal{P}}\|\ell(w; x, y)-\mathbb{E}_{(x,y)\sim \mathcal{P}}\ell(w; x, y)\|^2_2$. For simplicity of presentation, we omit the details here. 
\end{remark}

\subsection{Extension to Generalized Linear Convex Loss Functions}

In this section, we extend our results for the hinge loss function to generalized linear convex loss functions $L(w, D)=\frac{1}{n}\sum_{i=1}^n f(y_i\langle x_i, w\rangle)$ for any 1-Lipschitz convex function $f$. 

One possible way (for the extension) is to follow the same approach used in previous section. That is, we first smoothen the function $f$ by some function $f_\beta$. Then, we use Bernstein polynomials to approximate the derivative function $f'_\beta$, and apply an algorithm similar to Algorithm \ref{alg:1}. One of the main issues of this approach is that we do not know whether Bernstein polynomials can be directly used for every smooth convex function. 
Instead, we will use
some ideas in Approximation Theory, which says that every $1$-Lipschitz convex function can be expressed by a linear combination of the absolute functions and some linear functions.

To implement this approach, we first note that 
for the plus function $f(x)\equiv \max\{0, x\}$, by using Algorithm \ref{alg:1} we can get the same result as in Theorem \ref{theorem:3}. Since 
the absolute function $|x|=2\max\{0, x\}-x$,  Theorem \ref{theorem:3} clearly also holds for the absolute function. The following key lemma shows that every 1-dimensional $1$-Lipschitz convex function $f:[-1,1]\mapsto [-1, 1]$ is contained in the convex hull of the set of absolute and identity functions. We need to point out that \citet{smith2017interaction} gave a similar lemma. Their proof is, however, somewhat incomplete and thus we give a complete one in this paper.

 \begin{lemma}\label{lemma:4}
Let $f: [-1, 1]\mapsto [-1,1] $ be a 1-Lipschitz convex function. If we define the distribution $\mathcal{Q}$  which is supported on $[-1, 1]$ as the output of the following algorithm:  
\begin{enumerate}
\item first sample $u\in [f'(-1), f'(1)]$ uniformly, 
\item then output $s$ such that $u\in \partial{f}(s)$ (note that such an $s$ always exists due to the fact that $f$ is convex and thus $f'$ is non-decreasing); if  multiple number of such as $s$ exist, 
return the maximal one,
\end{enumerate}
then, there exists a constant $c$ such that 
\begin{equation*}
\forall \theta\in [-1,1], f(\theta)=\frac{f'(1)-f'(-1)}{2}\mathbb{E}_{s\sim \mathcal{Q}}|\theta-s|+\frac{f'(1)+f'(-1)}{2}\theta+c.
\end{equation*}
 
\end{lemma}

\begin{proof}[\textbf{Proof of Lemma \ref{lemma:4}}]
Let $g(\theta) =\mathbb{E}_{s\sim \mathcal{Q}}|s-\theta|$. Then, we have the following for every $\theta$, where $f'(\theta)$ is well defined,
\begin{align*}
 g'(\theta)&=\mathbb{E}_{s\sim \mathcal{Q}}[1_{s\leq \theta}]-\mathbb{E}_{s\sim \mathcal{Q}}[1_{s>\theta}]\\
&=\frac{[f'(\theta)-f'(-1)]-[f'(1)-f'(\theta)]}{f'(1)-f(-1)}\\
&= \frac{2f'(\theta)-(f'(1)+f'(-1))}{f'(1)-f'(-1)}.
\end{align*}
Thus, we get 
\begin{equation*}
F'(\theta)=\frac{f'(1)-f'(-1)}{2}g'(\theta)+\frac{f'(1)+f'(-1)}{2}=f'(\theta).
\end{equation*}
Next, we show that if $F'(\theta)=f'(\theta)$ for every $\theta\in [0,1]$, where $f'(\theta)$ is well defined,  there is a constant $c$ which satisfies the condition of  $F(\theta)=f(\theta)+c$ for all $\theta\in[0,1]$.

\begin{lemma}\label{lemma:5}
If $f$ is convex and 1-Lipschitz, then $f$ is differentiable at all but countably many points. That is, $f'$  has only countable many discontinuous points.
\end{lemma}
\begin{proof}[Proof of Lemma \ref{lemma:5}]
Since $f$ is convex, we have the following for $0\leq s<u\leq v<t\leq 1$  
\begin{equation*}
\frac{f(u)-f(s)}{u-s}\leq \frac{f(t)-f(v)}{t-v},
\end{equation*}
This is due to the property of 3-point convexity, where 
\begin{equation*}
\frac{f(u)-f(s)}{u-s}\leq \frac{f(t)-f(u)}{t-u}\leq  \frac{f(t)-f(v)}{t-v}.
\end{equation*}
Thus, we can obtain the following inequality of one-sided derivation, that is, 
\begin{equation*}
f'_{-}(x)\leq f'_+(x)\leq f'_-(y)\leq f'_+(y)
\end{equation*}
for every $x<y$. For each point where $f'_-(x)<f'_+(x)$, we pick a rational number $q(x)$ which satisfies the condition of $f'_-(x)<q(x)<f'_+(x)$.  From the above discussion, we can see that all these $q(x)$ are different. Thus, there are at most countable many points where $f$ is non-differentiable. 
\end{proof}

From the above lemma, we can see that the Lebesgue measure of these dis-continuous points is $0$. Thus, $f'$ is Riemann Integrable on $[-1, 1]$. By Newton-Leibniz formula, we have the following for any $\theta\in[0,1]$, $$\int_{-1}^\theta f'(x)dx=f(\theta)-f(-1)=\int_{-1}^\theta F'(x)dx= F(x)-F(-1).$$ Therefore, we get $F(\theta)=f(\theta)+c$ and complete the proof.
\end{proof}

\begin{algorithm}[h]
	\caption{General Linear-LDP}
	\label{alg:2}
	\begin{algorithmic}[1]
		\STATE {\bfseries Input:} Player $i\in [n]$ holds raw data record $(x_i, y_i)\in \mathcal{D}$, where $\|x_i\|_2\leq 1$ and $\|y_i\|_2\leq 1$; privacy parameters $\epsilon, \delta$;  $h_\beta(x)=\frac{x+\sqrt{x^2+\beta^2}}{2}$ and $P_d(x)=\sum_{j=0}^dc_j \binom{d}{j} x^j (1-x)^j$ is the $d$-th order Bernstein polynomial approximation of $h'_\beta(x)$. Loss function $\ell$ can be represented by $\ell(w; x, y)=f(y<w,x>)$.
\FOR{Each Player $i\in [n]$}
		\STATE
		Calculate $x_{i,0}=x_i+\sigma_{i,0}$ and $y_{i,0}=y_i+z_{i,0}$,  where $\sigma_{i,0} \sim \mathcal{N}(0, \frac{32\log (1.25/\delta)}{\epsilon^2}I_{p})$ and $ z_{i,0}\sim \mathcal{N}(0, \frac{32\log (1.25/\delta)}{\epsilon^2})$
		\FOR{$j=1,\cdots, d(d+1)$}
		    \STATE $x_{i, j}=x_i+ \sigma_{i, j}$, where $\sigma_{i, j}\sim \mathcal{N}(0, \frac{8\log(1.25/\delta)d^2(d+1)^2}{\epsilon^2}I_{p})$
		    \STATE $y_{i, j}=y_i+ z_{i, j}$, where $z_{i, j}\sim \mathcal{N}(0, \frac{8\log(1.25/\delta)d^2(d+1)^2}{\epsilon^2})$
	
		\ENDFOR
		\STATE Send $\{x_{i,j}\}_{j=0}^{d(d+1)}$ and $\{y_{i, j}\}_{j=0}^{d(d+1)}$ to the server.
		\ENDFOR
\FOR{the Server side}
	\FOR{$t=1, 2, \cdots, n$}
	\STATE Randomly sample $i\in [n]$ uniformly.
    \STATE Randomly sample $d(d+1)$ numbers of  i.i.d $s=\{s_k\}_{k=1}^{d(d+1)}\in [-1, 1]$ based on the distribution $\mathcal{Q}$ in  Lemma \ref{lemma:4}.
	\STATE Set $t_{i,0}=1$
	\FOR{$j=0, \cdots, d$}
	\STATE $t_{i, j}=\Pi_{k=jd+1}^{jd+i}(\frac{	y_{i, k}<w_t, x_{i, k}>-s_k}{2})$ and $t_{i, 0}=1 $
	\STATE $r_{i, j}=\Pi_{k=jd+i+1}^{jd+d}(1- \frac{	y_{i, k}<w_t, x_{i, k}>-s_k}{2})$ and $r_{i, d}=1 $
	\ENDFOR
	\STATE Denote $G(w_t, i, s)=(f'(1)-f'(-1))(\sum_{j=0}^{d}c_j\binom{d}{j} t_{i, j}r_{i,j})y_{i, 0}x^T_{i, 0}+f'(-1)$.
	\STATE Update SIGM in \citep{dvurechensky2016stochastic} by $G(w_t, i, s)$
	\ENDFOR
	\ENDFOR\\
	\Return $w_n$
	\end{algorithmic}
\end{algorithm}

 Using Lemma \ref{lemma:4} and the ideas discussed in the previous section, we can now show that the sample complexity in Theorem \ref{theorem:3} also holds for any general linear convex function. See Algorithm \ref{alg:2} for the details. 

\vspace{-0.2cm}
\begin{theorem}\label{theorem:4}
Under Assumption 1, where the loss function $\ell$ is $\ell(w; x, y)=f(y \langle w,x \rangle )$ for any 1-Lipschitz convex function $f$,  for any $\epsilon, \delta\in (0,1]$, Algorithm \ref{alg:2} is $(\epsilon, \delta)$ non-interactively differentialy private. Moreover, given the target  error $\alpha$,  if we take $\beta=\frac{\alpha}{4}$ and $d=\frac{2}{\beta^2\alpha}=O(\frac{1}{\alpha^3})$. Then with the sample size $n=\tilde{O}(\frac{d^{6d}C^d p}{\epsilon^{4d+4}\alpha^2})$, the output $w_n$ satisfies the following inequality 
$$\mathbb{E} L(w_n, D)-\min_{w\in\mathcal{C}}L(w, D)\leq \alpha,$$
where $C$ is  some universal constant independent of $f$.
\end{theorem}
\begin{proof}[\textbf{Proof of Theorem \ref{theorem:4}}]

 Let $h_\beta$ denote the function  $h_\beta(x)=\frac{x+\sqrt{x^2+\beta^2}}{2}$. By Lemma \ref{lemma:4} we have  
\begin{align*}
f(\theta)&=(f'(1)-f'(-1))\mathbb{E}_{s\sim \mathcal{Q}}\frac{|s-\theta|}{2}+\frac{f'(1)+f'(-1)}{2}\theta+c.
\end{align*}
Now, we consider  function $F_\beta(\theta)$, which is 
\begin{equation*}
F_\beta(\theta)= (f'(1)-f'(-1))\mathbb{E}_{s\sim \mathcal{Q}}[2h_\beta(\frac{\theta-s}{2})-\frac{\theta-s}{2}]+\frac{f'(1)+f'(-1)}{2}\theta+c.
\end{equation*}
From this, we  have 
\begin{equation*}
\nabla F_\beta(\theta)= (f'(1)-f'(-1))\mathbb{E}_{s\sim \mathcal{Q}}[\nabla h_\beta(\frac{\theta-s}{2})]+\frac{f'(1)+f'(-1)}{2}-\frac{f'(1)-f'(-1)}{2}.
\end{equation*}
Note that since $|x|=2\max\{x, 0\}-x$, we can get 1) $|F_\beta(\theta)-f(\theta)|\leq O(\beta)$ for any $\theta\in \mathbb{R}$, 2) 
$F_\beta(x)$ is $O(\frac{1}{\beta})$-smooth and convex since $h_\beta(\theta-s)$ is $\frac{1}{\beta}$-smooth and convex, and 3) $F_\beta(\theta)$ is $O(1)$-Lipschitz. Now, we optimize the following problem in the non-interactive local model:
\begin{equation*}
F_\beta(w; D)=\frac{1}{n}\sum_{i=1}^n F_\beta(y_i<x_i, w>).
\end{equation*}
For each fixed $i$ and $s$, we let
\begin{equation*}
    \hat{G}(w, i, s)=(f'(1)-f'(-1))[\sum_{j=1}^dc_j\binom{d}{j}t_{i,j}r_{i,j }]y_ix_i^T+f'(-1).
\end{equation*}
Then, we have $\mathbb{E}_{\sigma, z}G(w, i, s)=\hat{G}(w, i, s)$. By using a similar argument given in the proof of Theorem \ref{theorem:2}, we get 
\begin{equation*}
    \text{Var}(\hat{G}(w, i, s)|i, s)\leq \tilde{O}\big( \frac{d^{6d}C^d p}{\epsilon^{4d+4}}\big).
\end{equation*}
 Thus,  for each fixed $i$ we have 
 \begin{multline*}
 	\mathbb{E}_{s}\hat{G}(w, i, s)=\bar{G}(w,i)= (f'(1)-f'(-1))[\mathbb{E}_{s\sim \mathcal{Q}}\sum_{j=1}^dc_j\binom{d}{j} (\frac{y_i\langle w,x_i \rangle -s}{2})^j\\(1-\frac{y_i\langle w,x_i\rangle -s}{2})^{d-j}]y_ix_i^T+f'(-1).
 \end{multline*}
 
Next, we  bound the term of $\text{Var}(\hat{G}(w, i, s)|i)\leq O(d^{2d+2}).$

Let $t_{i,j}=\Pi_{k=jd+1}^{jd+j}(\frac{	y_{i}\langle w_t, x_{i}\rangle -s_k}{2})$. Then, we have $$\text{Var}(t_{i,j})\leq \Pi_{k=jd+1}^{jd+j}|y_i|^2\text{Var}(\langle w_t, x_{i}\rangle -s_k)\leq O(1).$$ And the similar to $r_{i,j}$
Thus, we get  
\begin{equation*}
    \text{Var}(\hat{G}(w, i, s)|i)\leq O(\sum_{j=1}^d c_j^2\binom{d}{j} \text{Var}(t_{i,j}r_{i,j}))=O(d^{2d+2}).
\end{equation*}
Since $\mathbb{E}_i\bar{G}(w,i)=\hat{G}=\frac{1}{n}\sum_{i=1}^n \bar{G}(w,i)$, we have $\text{Var}(\bar{G}(w,i))\leq O((\alpha+1)^2)$ by a similar argument given in the proof of Theorem \ref{theorem:2}.
Thus, in total we have 
\begin{equation*}
    \mathbb{E}\|G(w, i, s)-\hat{G}\|_2^2\leq  \tilde{O}\big( \frac{d^{6d}C^d p}{\epsilon^{4d+4}}\big)
\end{equation*}
The other part of the proof is the same as that of Theorem \ref{theorem:2}.
\end{proof}
\begin{remark}
The above theorem suggests that 
the sample complexity for any generalized linear loss function depends only linearly on $p$. However, there are still some not so desirable issues. Firstly, the dependence on $\alpha$ is quasi-polynomial, while  previous work  \citep{DBLP:journals/corr/abs-1802-04085} has already shown that it is only polynomial ({\em i.e.,} $\alpha^{-4}$) for sufficiently smooth loss functions. Secondly,  the term of $\epsilon$ is not optimal in the sample complexity, since it is $\epsilon^{-\Omega(1/\alpha^3)}$, while the optimal one is $\epsilon^{-2}$. We leave it as an open problem to remove the quasi-polynomial dependency.  Thirdly, the assumption on the loss function is that  $\ell(w; x, y)=f(y \langle w,x \rangle)$, which includes the generalized linear models and SVM. However, as mentioned earlier, there is another slightly more general function class 
$\ell(w; x, y)=f(\langle w,x \rangle,y)$ which does not always satisfy our assumption, {\em e.g.,} linear regression and $\ell_1$ regression. For linear regression, we have already known its optimal bound $\Theta(p\alpha^{-2}\epsilon^{-2})$;  for $\ell_1$ regression, we can use a method similar to Algorithm \ref{alg:1} to achieve a sample complexity which is linear in $p$. Thus, a natural question is whether the sample complexity is still
 linear in $p$ for all loss functions $\ell(w; x, y)$ that can be written as  $f(\langle w,x \rangle,y)$.
\end{remark}
\section{Discussion}
In this paper, we propose a general method for Empirical Risk Minimization in non-interactive differentially private model by using polynomial of inner product approximation. Compared with the method of directly using polynomial approximation, such as the one in \citep{DBLP:journals/corr/abs-1802-04085}, which needs exponential (in $p$) number of grids to estimate the function privately, our method avoid this undesirable issue. Using this method, we show that the sample complexity for any $1$-Lipschtiz generalized linear convex function is only linear in $p$. 

\acks{D.W. and J.X. were supported in part by NSF through grants CCF-1422324 and CCF-1716400. A.S. was supported by NSF awards IIS-1447700 and AF-1763786 and a Sloan Foundation Research Award. Part of this research was done while D.W. was visiting Boston University and Harvard University's Privacy Tools Project.}

\bibliography{alt_19}

\end{document}